\documentclass[conference]{IEEEtran}
\IEEEoverridecommandlockouts

\usepackage{cite}
\usepackage{amsmath,amssymb,amsfonts}
\usepackage{amsthm}
\usepackage{algorithm}
\usepackage{algpseudocode}

\newtheorem{theorem}{Theorem}
\usepackage{graphicx}
\usepackage{textcomp}
\usepackage{xcolor}
\usepackage{booktabs}
\usepackage{array}

\def\BibTeX{{\rm B\kern-.05em{\sc i\kern-.025em b}\kern-.08em
    T\kern-.1667em\lower.7ex\hbox{E}\kern-.125emX}}

\begin{document}

\title{Enhancing SignSGD: Small-Batch Convergence Analysis and a Hybrid Switching Strategy}

\author{\IEEEauthorblockN{Haoran Chen, Wentao Wang}
\IEEEauthorblockA{\textit{École polytechnique} \\
Palaiseau, France \\
haoran.chen@polytechnique.edu, wentao.wang@polytechnique.edu}
}

\maketitle

\begin{abstract}
SignSGD compresses each stochastic gradient coordinate to a single bit, offering substantial memory and communication savings, but its 1-bit quantization removes magnitude information and is known to leave a generalization gap relative to well-tuned SGD. We revisit SignSGD from a 1-bit quantization and dithering perspective and contribute three improvements. First, we derive a small-batch convergence rate for SignSGD under unimodal symmetric gradient noise using a signal-to-noise weighted stationarity measure, removing the large-batch assumption of prior analyses. Second, we inject annealed Gaussian noise before the sign operator, which acts as a classical dithering mechanism and probabilistically restores magnitude information lost to hard thresholding. Third, we adapt the SWATS strategy to sign-based updates with a projection-based learning-rate calibration that smoothly transitions from SignSGD to SGD. Single-worker experiments on ResNet-18 isolate optimizer effects from communication aspects: pre-sign dithering surpasses Adam on CIFAR-100, and the calibrated switch reaches 92.18\% test accuracy on CIFAR-10, outperforming both pure SGD (91.38\%) and pure SignSGD with momentum (90.82\%).
\end{abstract}


\section{Introduction}
Training deep neural networks at scale stresses both compute and communication budgets, motivating compressed gradient updates for memory-constrained and distributed settings. SignSGD~\cite{bernstein2018signsgd} replaces each stochastic gradient coordinate $\tilde{g}_i$ by its sign, yielding a 1-bit update with strong communication-reduction properties when combined with majority-vote aggregation across workers. From a signal-processing viewpoint, $\mathrm{sign}(\tilde{g}_i)$ is a coarse 1-bit quantizer applied to a noisy signal, and many of the algorithm's behaviors---robustness to sparse gradient noise, biased fixed points, and sensitivity to small-magnitude coordinates---are natural consequences of this quantization.

Despite its appeal, sign-based optimizers exhibit a persistent generalization gap relative to well-tuned SGD on vision benchmarks~\cite{bernstein2018signsgd}, mirroring the gap observed for adaptive methods more broadly~\cite{keskar2017swats}, and now also visible in modern variants such as Lion~\cite{chen2023lion}, which is in essence a sign-of-momentum optimizer. The gap is commonly attributed to two factors: (i) hard thresholding discards magnitude information that would otherwise modulate the step direction near small or noisy coordinates, and (ii) the uniform $\pm\delta$ step can prevent convergence to flat minima associated with good generalization~\cite{karimireddy2019error}.

\textbf{Contributions.} Through a 1-bit-quantization-and-dithering lens, we make three contributions:
\begin{itemize}
\item \textit{Small-batch convergence analysis} (Sec.~\ref{sec:convergence}): we derive a signal-to-noise (SNR) weighted stationarity rate for SignSGD with constant batch size under unimodal symmetric noise, complementing the large-batch rate of~\cite{bernstein2018signsgd}.
\item \textit{Pre-sign dithering} (Sec.~\ref{sec:dithering}): we add annealed Gaussian noise before the sign operator, the optimizer-side analogue of classical subtractive dithering in 1-bit A/D conversion~\cite{schuchman1964dither, gray1993dithered}, and show it consistently improves test accuracy on CIFAR-100, surpassing Adam.
\item \textit{Calibrated switch to SGD} (Sec.~\ref{sec:switching}): we adapt SWATS~\cite{keskar2017swats} to sign-based updates with a projection-based learning-rate calibration that matches the effective step size of SignSGD with momentum at the switching iterate, avoiding the magnitude mismatch of a naive switch.
\end{itemize}

\textbf{Scope.} All experiments use a single worker and are intended to isolate the optimization-side effects of 1-bit quantization, dithering, and switching from communication-side considerations. The dithering and switching mechanisms are compatible with majority-vote aggregation, but a multi-worker evaluation is left to future work.

\section{Background and Related Work}

\textbf{SignSGD and variants.} Bernstein et al.~\cite{bernstein2018signsgd} established that SignSGD attains an $\mathcal{O}(1/\sqrt{N})$ rate on the average $\ell_1$-gradient norm under coordinate-wise smoothness and bounded variance, and that majority-vote aggregation across $M$ workers reduces the variance term by $\sqrt{M}$. Subsequent work showed that vanilla SignSGD can fail to converge under non-symmetric noise and proposed error-feedback variants~\cite{karimireddy2019error}; communication-compressed Adam variants such as 1-bit Adam~\cite{tang2021onebitadam} extend the same idea to adaptive optimizers. Lion~\cite{chen2023lion}, recently popular for large-scale pretraining, is in essence a sign-of-momentum optimizer and inherits the geometry studied here. An early predecessor is 1-bit SGD applied to speech DNN training~\cite{seide2014onebit}.

\textbf{Quantized and compressed SGD.} Beyond 1-bit, QSGD~\cite{alistarh2017qsgd}, TernGrad~\cite{wen2017terngrad}, and Top-K sparsification trade compression rate against unbiasedness. SignSGD sits at the extreme compression end and trades unbiasedness for simplicity.

\textbf{Gradient noise and dithering.} Adding Gaussian noise to gradients improves training of very deep networks~\cite{neelakantan2015adding}; in signal processing, dithering before a coarse quantizer is a classical mechanism to decorrelate quantization error from the input and to encode sub-quantum-step information probabilistically~\cite{schuchman1964dither, gray1993dithered}. Our pre-sign noise injection makes this connection explicit on the optimizer side.

\textbf{Adaptive-to-SGD switching.} SWATS~\cite{keskar2017swats} starts with Adam and switches to SGD with a learning rate determined by a projection condition. We transplant this idea to sign-based updates with a momentum-aware projection.

\section{Small-Batch Convergence Under Unimodal Symmetric Noise}
\label{sec:convergence}

We adopt the standard assumptions of~\cite{bernstein2018signsgd}: $f$ is lower bounded by $f^\star$, has a coordinate-wise Lipschitz constant vector $\vec{L}$ with $L_i$ the Lipschitz constant of $\partial_i f$, and the unbiased single-sample stochastic gradient oracle has coordinate variance bounded by $\sigma_i^2$. We additionally assume that the per-coordinate noise is unimodal and symmetric, so that the sign-failure-probability bound of~\cite{bernstein2018signsgd}---a consequence of Gauss's inequality~\cite{pukelsheim1994three}---applies coordinate-wise. The mini-batch estimator $\tilde{g}_k$ at iteration $k$ is the average of $n$ independent samples of the oracle, so that $\mathrm{Var}(\tilde{g}_{k,i} \mid x_k) \leq \sigma_i^2/n$.

Let $\tilde{L}_1 := \sum_i L_i$, $\tilde{\sigma}_1 := \sum_i \sigma_i$, $s_{k,i} := \sqrt{\mathrm{Var}(\tilde{g}_{k,i} \mid x_k)}$, and $S_{k,i} := |g_{k,i}|/s_{k,i}$ denote the per-coordinate signal-to-noise ratio at iteration $k$. We define the SNR-weighted stationarity measure
\begin{equation}
\Phi_k := \sum_{i=1}^d |g_{k,i}| \min\!\left(1, S_{k,i}\right) = \sum_{i=1}^d \min\!\left(|g_{k,i}|, \frac{|g_{k,i}|^2}{s_{k,i}}\right).
\label{eq:phi}
\end{equation}
$\Phi_k$ smoothly interpolates between $\|g_k\|_1$ on high-SNR coordinates and a quadratic-in-$|g_{k,i}|$ measure on low-SNR coordinates, reflecting the diminishing usefulness of the sign bit when noise dominates the signal. This is a more refined stationarity criterion than $\|g_k\|_1$: it weights each coordinate by its informativeness rather than its raw magnitude.

\begin{theorem}[Small-batch SNR-weighted rate of SignSGD]
\label{thm:rate}
Run SignSGD for $K$ iterations with constant mini-batch size $n_k \equiv n$ and constant stepsize $\delta_k \equiv 1/\sqrt{\tilde{L}_1 K}$. Then
\begin{align}
\mathbb{E}\!\left[\frac{1}{K}\sum_{k=0}^{K-1} \Phi_k\right] &\leq \frac{3\sqrt{\tilde{L}_1}}{\sqrt{K}}\!\left(f_0 - f^\star + \tfrac{1}{2}\right), \label{eq:rate_phi}\\
\mathbb{E}\!\left[\frac{1}{K}\sum_{k=0}^{K-1} \|g_k\|_1\right] &\leq \frac{3\sqrt{\tilde{L}_1}}{\sqrt{K}}\!\left(f_0 - f^\star + \tfrac{1}{2}\right) + \frac{\tilde{\sigma}_1}{\sqrt{n}}. \label{eq:rate_l1}
\end{align}
\end{theorem}

\begin{proof}
By coordinate-wise smoothness with $y = x_{k+1}$, $x = x_k$,
\begin{equation}
f_{k+1} - f_k \leq -\delta_k g_k^\top \mathrm{sign}(\tilde{g}_k) + \frac{\delta_k^2}{2} \tilde{L}_1.
\label{eq:smooth}
\end{equation}
Let $p_{k,i} := \mathbb{P}[\mathrm{sign}(\tilde{g}_{k,i}) \neq \mathrm{sign}(g_{k,i}) \mid x_k]$, so that
\begin{equation}
\mathbb{E}[g_{k,i}\, \mathrm{sign}(\tilde{g}_{k,i}) \mid x_k] = |g_{k,i}|(1 - 2p_{k,i}).
\label{eq:sign_exp}
\end{equation}
The Gauss-inequality bound for unimodal symmetric noise~\cite{bernstein2018signsgd, pukelsheim1994three} gives, after a slight relaxation of the split point at $S = \sqrt{2/3}$,
\begin{equation}
p_{k,i} \leq
\begin{cases}
\dfrac{2}{9 S_{k,i}^2}, & S_{k,i} > \sqrt{2/3},\\[4pt]
\dfrac{1}{2} - \dfrac{S_{k,i}}{2\sqrt{3}}, & \text{otherwise},
\end{cases}
\label{eq:gauss}
\end{equation}
which yields $1 - 2p_{k,i} \geq \frac{1}{3}\min(1, S_{k,i})$ in both regimes. Summing over coordinates and combining with~\eqref{eq:sign_exp} gives
\begin{equation}
\mathbb{E}[g_k^\top \mathrm{sign}(\tilde{g}_k) \mid x_k] \geq \frac{1}{3}\Phi_k.
\label{eq:phi_lb}
\end{equation}
Substituting~\eqref{eq:phi_lb} into~\eqref{eq:smooth}, taking total expectation and telescoping over $k = 0, \ldots, K-1$,
\begin{equation}
f_0 - f^\star \geq \frac{1}{3}\sum_k \delta_k \mathbb{E}[\Phi_k] - \frac{\tilde{L}_1}{2}\sum_k \delta_k^2.
\label{eq:telescope}
\end{equation}
With $\delta_k \equiv 1/\sqrt{\tilde{L}_1 K}$, rearranging yields~\eqref{eq:rate_phi}. For~\eqref{eq:rate_l1}, the elementary inequality $|a| \leq \min(|a|, a^2/s) + s$ valid for $a \in \mathbb{R}$, $s > 0$, applied with $s_{k,i} \leq \sigma_i/\sqrt{n}$ gives $\|g_k\|_1 \leq \Phi_k + \tilde{\sigma}_1/\sqrt{n}$. Averaging over $k$ and combining with~\eqref{eq:rate_phi} completes the proof.
\end{proof}

\textbf{Discussion.} Compared to Theorem 1 of~\cite{bernstein2018signsgd}, which requires $n_k \equiv K$ (a large-batch regime where increasing the optimization horizon implicitly raises the batch size), Theorem~\ref{thm:rate} holds for any constant $n$, at the price of a non-vanishing $\tilde{\sigma}_1/\sqrt{n}$ floor in~\eqref{eq:rate_l1}. This floor is an honest reflection of the cost of 1-bit quantization at small batch sizes: when the per-coordinate SNR is low, the sign bit carries little information about the gradient direction, and no amount of iteration can drive $\|g_k\|_1$ below the noise level encoded in $\tilde{\sigma}_1/\sqrt{n}$. The SNR-weighted measure $\Phi_k$ in~\eqref{eq:rate_phi}, by contrast, vanishes at the standard $\mathcal{O}(1/\sqrt{K})$ rate, providing a clean stationarity guarantee in the small-batch regime.

\section{Pre-Quantization Dithering}
\label{sec:dithering}

\textbf{Mechanism.} The sign operator is a coarse 1-bit quantizer; its output ignores all magnitude information in its argument. Subtractive and non-subtractive dithering---adding controlled noise to a signal before quantization---is a textbook technique to decorrelate quantization error from the input and to encode sub-quantum-step information probabilistically~\cite{schuchman1964dither, gray1993dithered}. We apply this idea to SignSGD with momentum (SignSGD-M), summarized in Algorithm~\ref{alg:dither}.

\begin{algorithm}[t]
\caption{Dithered SignSGD-M (pre-sign noise)}
\label{alg:dither}
\begin{algorithmic}[1]
\Require learning rate $\delta$, momentum $\beta \in (0,1)$, dither schedule $\sigma_k^2 = \alpha(1+k)^{-\gamma}$
\State Initialize $x_0$, $m_0 \gets 0$
\For{$k = 0, 1, \ldots, K-1$}
    \State Sample mini-batch and compute $\tilde{g}_k$
    \State $m_{k+1} \gets \beta m_k + (1-\beta)\tilde{g}_k$
    \State Sample dither $\xi_k \sim \mathcal{N}(0, \sigma_k^2 I)$
    \State $x_{k+1} \gets x_k - \delta\, \mathrm{sign}(m_{k+1} + \xi_k)$
\EndFor
\end{algorithmic}
\end{algorithm}

We anneal the dither variance following~\cite{neelakantan2015adding}:
\begin{equation}
\sigma_k^2 = \alpha (1+k)^{-\gamma}, \qquad \gamma = 0.55.
\label{eq:anneal}
\end{equation}
For comparison we also evaluate \textit{post-sign} injection,
\begin{equation}
x_{k+1} = x_k - \delta\!\left(\mathrm{sign}(m_{k+1}) + \xi_k\right),
\label{eq:postsign}
\end{equation}
which perturbs the already-normalized step and is closer in spirit to exploration noise on a constant-magnitude update.

\textbf{Effect on the sign-failure probability.} Consider a single coordinate with momentum signal $m_i > 0$ and pre-sign Gaussian noise $\xi_i \sim \mathcal{N}(0, \sigma_k^2)$. The probability of correct sign is $\Phi(m_i/\sigma_k)$, where $\Phi$ is the standard normal CDF. As $\sigma_k \to 0$ this recovers the deterministic sign of $m_i$; for moderate $\sigma_k$, coordinates with $|m_i| \lesssim \sigma_k$ produce dithered (probabilistic) signs whose expectation is $2\Phi(m_i/\sigma_k) - 1 \approx m_i\sqrt{2/\pi}/\sigma_k$ to first order. This is, in spirit, the optimizer-side analogue of classical dithering theory~\cite{schuchman1964dither, gray1993dithered}: the average update on small-magnitude coordinates becomes proportional to $m_i$ rather than its sign, partially restoring the magnitude information lost to quantization. The annealing schedule~\eqref{eq:anneal} ensures that this effect is strongest early in training---when the gradient signal is informative but coarsely quantized---and decays so as not to interfere with late-stage convergence.

\begin{figure}[t]
\centerline{\includegraphics[width=\columnwidth]{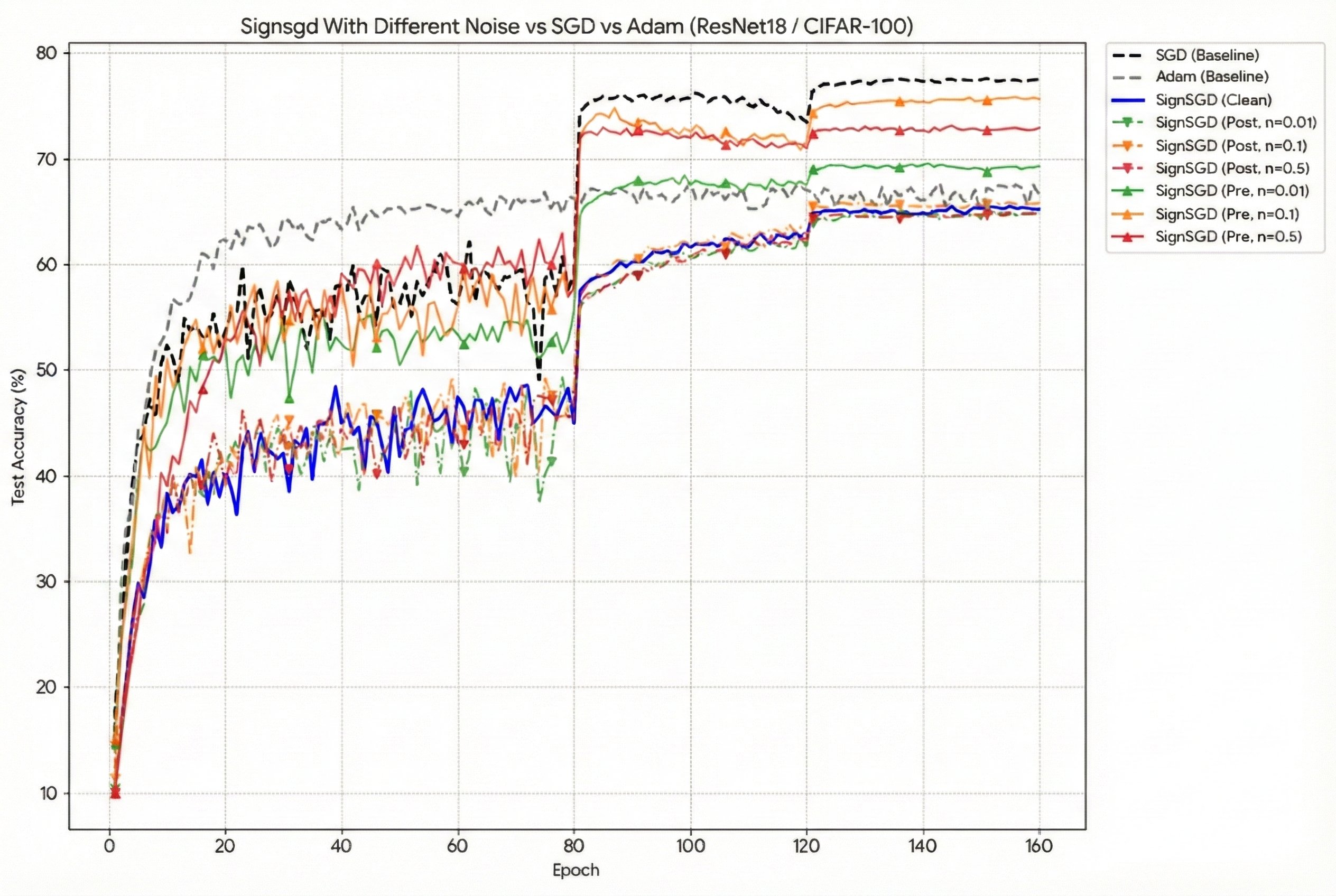}}
\caption{Test accuracy on CIFAR-100 (ResNet-18) for SignSGD-M with pre- and post-sign Gaussian dithering, against SGD and Adam baselines. Pre-sign injection with $\alpha = 0.1$ is the best sign-based variant, surpassing Adam and approaching SGD; post-sign injection is unstable. Stepwise jumps reflect a learning-rate schedule shared across all methods.}
\label{fig:noise}
\end{figure}

\textbf{Empirical results.} We evaluate on CIFAR-100 with ResNet-18~\cite{he2016resnet} (Fig.~\ref{fig:noise}). Pre-sign dithering with $\alpha = 0.1$ reaches roughly 76\% test accuracy, the best among all sign-based variants, surpassing Adam~\cite{kingma2015adam} ($\sim$67\%) and approaching the SGD baseline ($\sim$77\%). Larger $\alpha = 0.5$ also helps but plateaus lower, while $\alpha = 0.01$ provides marginal benefit, indicating that the dither variance must be matched to the typical magnitude of the momentum signal. Post-sign injection, which perturbs the already-normalized step, fails to provide a meaningful gain: it shows at most a marginal improvement over clean SignSGD-M (within 1\%) and degrades for larger $\alpha$, consistent with random-walk behavior on a fixed-magnitude update.

\section{Calibrated Switching to SGD}
\label{sec:switching}

\textbf{Motivation.} A naive hand-off from SignSGD-M to SGD at a fixed epoch is unstable: the normalized $\pm\delta$ updates of SignSGD-M and the gradient-proportional updates of SGD are on different magnitude scales, and reusing either learning rate produces stagnation or divergence. Following SWATS~\cite{keskar2017swats}, we calibrate the SGD learning rate at the switching iterate so that the SGD update matches the projection of the sign-based update onto the gradient direction.

\textbf{Projection rule.} Let $\delta$ be the SignSGD-M learning rate, $m_{k+1} = \beta m_k + (1-\beta)\tilde{g}_k$ the just-updated momentum at iteration $k$, and $\tilde{g}_k$ the current stochastic mini-batch gradient. The SignSGD-M step actually applied at iteration $k$ is $-\delta\,\mathrm{sign}(m_{k+1})$. We seek a scalar $\lambda_k \geq 0$ such that the SGD update $-\lambda_k \tilde{g}_k$ equals the projection of $-\delta\,\mathrm{sign}(m_{k+1})$ onto $\tilde{g}_k$:
\begin{equation}
\lambda_k \|\tilde{g}_k\|_2^2 = \delta\, \langle \mathrm{sign}(m_{k+1}), \tilde{g}_k\rangle.
\label{eq:proj}
\end{equation}
Enforcing a non-negative magnitude for stability,
\begin{equation}
\lambda_k = \delta\, \frac{|\langle \mathrm{sign}(m_{k+1}), \tilde{g}_k\rangle|}{\|\tilde{g}_k\|_2^2 + \epsilon},
\label{eq:lambda}
\end{equation}
with small $\epsilon > 0$ for numerical stability. During the SignSGD-M phase we maintain an exponential moving average (EMA) of $\lambda_k$ to reduce stochastic variance, and at a fixed switching epoch $T_{\mathrm{switch}}$ we transition to SGD with the EMA stepsize. We use a fixed $T_{\mathrm{switch}}$ to isolate the effect of the projection mechanism from the automatic-trigger heuristic of~\cite{keskar2017swats}. The full procedure is given in Algorithm~\ref{alg:hybrid}.

\begin{algorithm}[t]
\caption{Hybrid SignSGD-M $\to$ SGD with projection}
\label{alg:hybrid}
\begin{algorithmic}[1]
\Require SignSGD-M lr $\delta$, momentum $\beta \in (0,1)$, EMA decay $\eta$, switch epoch $T_{\mathrm{switch}}$
\State Initialize $x_0$, $m_0 \gets 0$, $\bar{\lambda} \gets 0$
\For{$k = 0, 1, \ldots, K-1$}
    \State Sample mini-batch and compute $\tilde{g}_k$
    \If{epoch $< T_{\mathrm{switch}}$}
        \State $m_{k+1} \gets \beta m_k + (1-\beta)\tilde{g}_k$
        \State $\lambda_k \gets \delta\, |\langle \mathrm{sign}(m_{k+1}), \tilde{g}_k\rangle|/(\|\tilde{g}_k\|_2^2 + \epsilon)$
        \State $\bar{\lambda} \gets \eta \bar{\lambda} + (1-\eta)\lambda_k$
        \State $x_{k+1} \gets x_k - \delta\, \mathrm{sign}(m_{k+1})$
    \Else
        \State $x_{k+1} \gets x_k - \bar{\lambda}\, \tilde{g}_k$ \Comment{SGD with calibrated lr}
    \EndIf
\EndFor
\end{algorithmic}
\end{algorithm}

\begin{figure}[t]
\centerline{\includegraphics[width=\columnwidth]{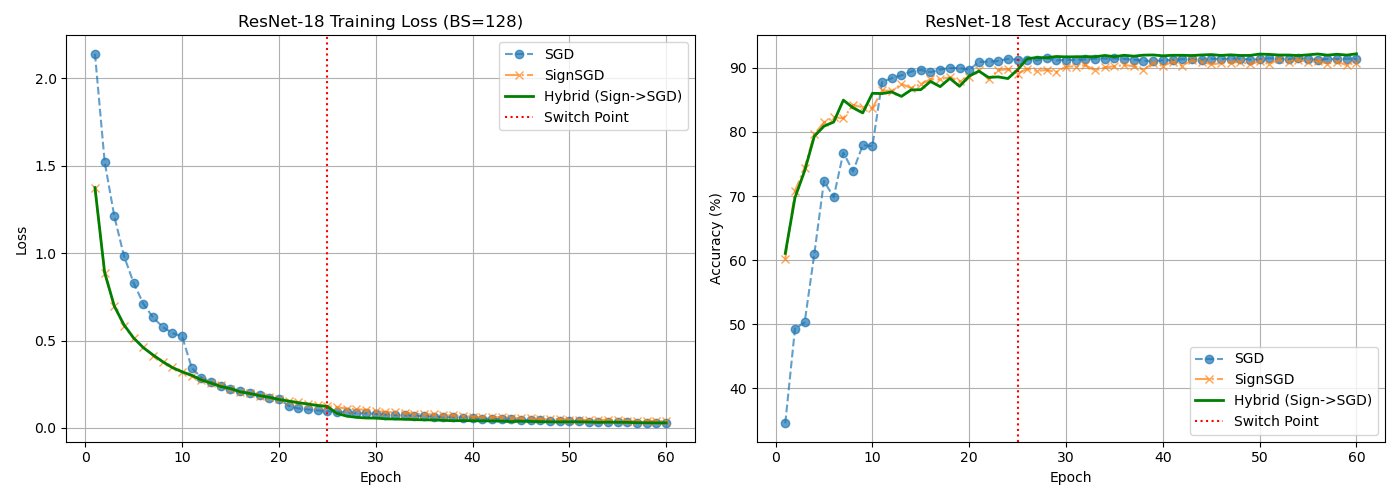}}
\caption{ResNet-18 on CIFAR-10. The hybrid method (green) follows the SignSGD-M trajectory through epoch 25 (red dashed line), then transitions to SGD with the projection-calibrated learning rate from~\eqref{eq:lambda}. Final test accuracy: hybrid 92.18\%, SGD 91.38\%, SignSGD-M 90.82\%.}
\label{fig:cifar10}
\end{figure}

\begin{figure}[t]
\centerline{\includegraphics[width=\columnwidth]{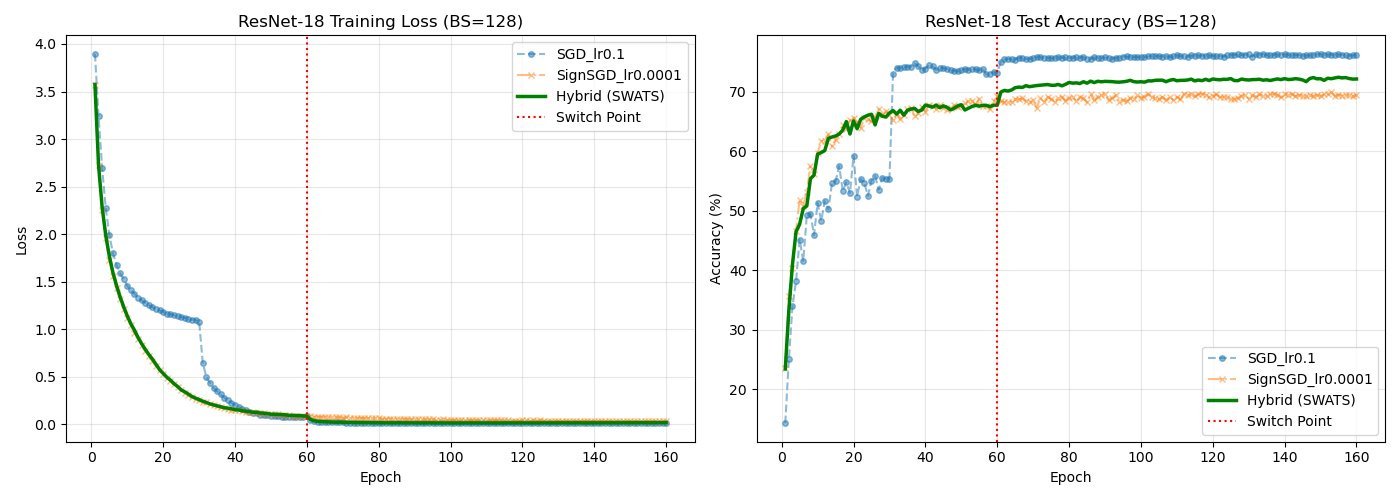}}
\caption{ResNet-18 on CIFAR-100. The hybrid method substantially mitigates the late-stage stagnation of SignSGD-M ($\sim$69\%), reaching $\sim$72\%, but does not match well-tuned SGD ($\sim$77\%) on this benchmark. Switch at epoch 60.}
\label{fig:cifar100}
\end{figure}

\textbf{Empirical results.} On CIFAR-10 (Fig.~\ref{fig:cifar10}), the hybrid method enjoys the rapid early progress of SignSGD-M---crossing 80\% test accuracy several epochs earlier than pure SGD---and after the switch fine-tunes smoothly to 92.18\%, outperforming both pure SGD (91.38\%) and pure SignSGD-M (90.82\%). On CIFAR-100 (Fig.~\ref{fig:cifar100}), the hybrid method substantially reduces the stagnation gap left by SignSGD-M ($\sim$69\% $\to$ $\sim$72\%), but does not close the gap to well-tuned SGD ($\sim$77\%); we attribute this to the larger generalization sensitivity of CIFAR-100 to the precise late-training trajectory and to the heuristic choice of $T_{\mathrm{switch}}$. The CIFAR-100 result indicates that calibrated switching is best understood as a remedy for sign-induced stagnation rather than a uniform improvement over SGD.

\textbf{Practical considerations.} The projection rule~\eqref{eq:lambda} requires only inner products and norms of quantities that SignSGD-M already computes, so the EMA tracking adds negligible cost. The dithering scheme of Sec.~\ref{sec:dithering} can be combined with the switching scheme of Sec.~\ref{sec:switching}: dithering operates inside the SignSGD-M phase and does not modify the projection rule. We did not observe instability from this combination in preliminary experiments, although a systematic study is beyond the scope of this paper.

\section{Results Summary and Discussion}

Table~\ref{tab:summary} consolidates the final test accuracies on ResNet-18 across all settings considered. Three observations stand out. First, plain SignSGD-M trails well-tuned SGD by roughly 0.6\% on CIFAR-10 and by approximately 12\% on CIFAR-100, confirming that the generalization cost of 1-bit quantization grows with task difficulty. Second, pre-sign dithering is the only one of our two improvements that closes most of the CIFAR-100 gap to SGD without changing the optimizer family, while post-sign injection provides at most a marginal benefit over the clean baseline and destabilizes training at higher noise levels---a clear empirical separation between dithering before and after the quantizer. Third, the calibrated switch is most effective on CIFAR-10, where the late-training landscape is benign enough that a short SignSGD-M warm-up followed by SGD beats either pure method; on CIFAR-100 the switch removes SignSGD-M's stagnation but does not match SGD, suggesting that the projection rule alone cannot recover the implicit regularization that SGD enjoys throughout training. Combining dithering during the SignSGD-M phase with the calibrated switch is a natural next step that we did not exhaustively explore here.

\begin{table}[t]
\caption{Final test accuracy (\%) on ResNet-18, single-worker training. Dithering rows (Sec.~\ref{sec:dithering}) and hybrid row (Sec.~\ref{sec:switching}) come from independent runs with matched architectures but separately tuned learning-rate schedules; CIFAR-100 sign-based numbers are read from Fig.~\ref{fig:noise} and rounded to the nearest integer. Best per column in \textbf{bold}.}
\label{tab:summary}
\centering
\begin{tabular}{lcc}
\toprule
Method & CIFAR-10 & CIFAR-100 \\
\midrule
SGD (well-tuned) & 91.38 & \textbf{77} \\
Adam & -- & 67 \\
SignSGD-M (clean) & 90.82 & 65 \\
SignSGD-M + post-noise ($\alpha{=}0.1$) & -- & 66 \\
SignSGD-M + pre-noise ($\alpha{=}0.5$) & -- & 73 \\
SignSGD-M + pre-noise ($\alpha{=}0.1$) & -- & 76 \\
Hybrid (SignSGD-M $\to$ SGD) & \textbf{92.18} & 72 \\
\bottomrule
\end{tabular}
\end{table}

\section{Conclusion}

We revisited SignSGD through a 1-bit-quantization-and-dithering lens and contributed (i) a small-batch SNR-weighted convergence rate under unimodal symmetric noise, removing the large-batch assumption of prior analyses; (ii) annealed pre-sign Gaussian dithering, the optimizer-side analogue of classical dither, which restores magnitude information lost to hard thresholding; and (iii) a SWATS-style calibrated switch from sign-based updates to SGD via a momentum-aware projection rule. Single-worker experiments on ResNet-18 isolate optimizer effects from communication and show that pre-sign dithering surpasses Adam on CIFAR-100, while the calibrated switch outperforms both pure SGD and SignSGD-M on CIFAR-10 and mitigates SignSGD-M stagnation on CIFAR-100 without fully closing the gap to SGD. Multi-worker experiments under majority-vote aggregation, error-feedback combinations, and an automatic switching trigger calibrated by the EMA stability of $\lambda_k$ are natural next steps.

\end{document}